\documentclass{article}

     \usepackage[final,nonatbib]{neurips_2019}

\usepackage[utf8]{inputenc} 
\usepackage[T1]{fontenc}    
\usepackage{hyperref}       
\usepackage{url}            
\usepackage{booktabs}       
\usepackage{amsfonts}       
\usepackage{nicefrac}       
\usepackage{microtype}      
\hypersetup{
    colorlinks = true,
 	linkcolor = teal,
 	citecolor = teal
}

\usepackage{xcolor,color}
\usepackage{soul}
\usepackage{enumerate}
\usepackage{indentfirst}
\usepackage{mathtools}
\usepackage{amsthm,amsfonts,amssymb}
\usepackage{subcaption}
\usepackage{thmtools}
\usepackage{tikz}
\usepackage[square,numbers]{natbib}
\usepackage[capitalise]{cleveref}

\newtheorem{thm}{Theorem}
\newtheorem{lem}{Lemma}

\theoremstyle{definition}
\newtheorem{defn}{Definition}
\newtheorem{example}{Example}

\newcommand \reals {\mathbb{R}}
\newcommand \cX {\mathcal{X}}
\newcommand \cY {\mathcal{Y}}

\newcommand \cH {\mathcal{H}}
\newcommand \cG {\mathcal{G}}
\newcommand \prob {\operatorname*{Pr}}
\newcommand \expect {\operatorname*{\mathbb{E}}}
\newcommand \ind [1]{\mathbb{I}\{#1\}}
\newcommand \norm [1]{\Vert#1\Vert}
\newcommand \resto {\bigl\vert}
\newcommand \dist {\operatorname{dist}}
\newcommand \NN {\operatorname{NN}}
\newcommand \tp {\top}

\newcommand \weights {\eta}

\title{Envy-Free Classification}

\author{Submission ID {\color{red}\#XXX}}

\author{%
  Maria-Florina Balcan \\
  Machine Learning Department\\
  Carnegie Mellon University\\
  \texttt{ninamf@cs.cmu.edu}\\
  \And
  Travis Dick\\
  Computer and Information Science\\
  University of Pennsylvania\\
  \texttt{tbd@seas.upenn.edu}\\
  \AND
  Ritesh Noothigattu\\
  Machine Learning Department\\
  Carnegie Mellon University\\
  \texttt{riteshn@cmu.edu}\\
  \And
  Ariel D. Procaccia\\
  Computer Science Department\\
  Harvard University\\
  \texttt{arielpro@seas.harvard.edu}\\
}

\begin{document}

\maketitle

\begin{abstract}
In classic fair division problems such as cake cutting and rent division, \emph{envy-freeness} requires that each individual (weakly) prefer his allocation to anyone else's. On a conceptual level, we argue that envy-freeness also provides a compelling notion of fairness for classification tasks, especially when individuals have heterogeneous preferences. Our technical focus is the \emph{generalizability} of envy-free classification, i.e., understanding whether a classifier that is envy free on a sample would be almost envy free with respect to the underlying distribution with high probability. Our main result establishes that a small sample is sufficient to achieve such guarantees, when the classifier in question is a mixture of deterministic classifiers that belong to a family of low Natarajan dimension.
\end{abstract}

\section{Introduction} \label{sec:intro}

The study of fairness in machine learning is driven by an abundance of examples
where learning algorithms were perceived as discriminating against protected
groups~\citep{Sween13,DTD15}. Addressing this problem requires a conceptual ---
perhaps even philosophical --- understanding of what fairness means in this
context. In other words, the million dollar question is
(arguably\footnote{Certain papers take a somewhat different
view~\citep{KRPH+17}.}) this: What are the formal constraints that fairness
imposes on learning algorithms?

In this paper, we propose a new measure of algorithmic fairness. It draws on an extensive body of work on
rigorous approaches to fairness, which --- modulo one possible exception (see
Section~\ref{subsec:rel}) --- has not been tapped by machine learning
researchers: the literature on \emph{fair division}~\citep{BT96,Moul03}. The
most prominent notion is that of \emph{envy-freeness}~\citep{Fol67,Var74},
which, in the context of the allocation of goods, requires that the utility of
each individual for his allocation be at least as high as his utility for the
allocation of any other individual; for six decades, it has been the gold standard of fairness for
problems such as cake cutting \citep{RW98,Pro13} and rent division
\citep{Su99,GMPZ17}. In the classification setting, envy-freeness would simply
mean that the utility of each individual for his distribution over outcomes is
at least as high as his utility for the distribution over outcomes assigned to
any other individual.

It is important to say upfront that envy-freeness is \emph{not} suitable for several widely-studied problems where there are only two possible outcomes, one of which is `good' and the other `bad'; examples include predicting whether an individual would default on a loan, and whether an offender would recidivate. In these degenerate cases, envy-freeness would require that the classifier assign each and every individual the exact same probability of obtaining the `good' outcome, which, clearly, is not a reasonable constraint.

By contrast, we are interested in situations where there is a diverse set of possible outcomes, and individuals have diverse preferences for those outcomes. For example, consider a system responsible for displaying credit card
advertisements to individuals. There are many credit cards with different
eligibility requirements, annual rates, and reward programs. An individual's
utility for seeing a card's advertisement will depend on his eligibility,
his benefit from the rewards programs, and potentially other factors. It may
well be the case that an envy-free advertisement assignment shows Bob
advertisements for a card with worse annual rates than those shown to Alice; this outcome is not unfair if Bob is genuinely more interested in the card
offered to him. Such rich utility functions are also evident in the context of
job advertisements~\citep{DTD15}: people generally want higher paying jobs, but
would presumably have higher utility for seeing advertisements for jobs that
better fit their qualifications and interests.

A second appealing property of envy-freeness is that its fairness guarantee
binds at the level of individuals. Fairness notions can be coarsely
characterized as being either individual notions, or group notions, depending on
whether they provide guarantees to specific individuals, or only on average to a
protected subgroup. The majority of work on fairness in machine learning focuses
on group fairness~\citep{LRT11,DHPR+12,ZWSP+13,MPS16,JKMA16,ZVGG+17}.


There is, however, one well-known example of individual fairness: the influential fair
classification model of \citet{DHPR+12}. The model involves a set of individuals
and a set of outcomes. The centerpiece of the model is a \emph{similarity
metric} on the space of individuals; it is specific to the classification task
at hand, and ideally captures the ethical ground truth about relevant
attributes. For example, a man and a woman who are similar in every other way
should be considered similar for the purpose of credit card offerings, but
perhaps not for lingerie advertisements. Assuming such a metric is available,
fairness can be naturally formalized as a Lipschitz constraint, which requires
that individuals who are close according to the similarity metric be mapped to
distributions over outcomes that are close according to some standard metric
(such as total variation).

As attractive as this model is, it has one clear weakness from a practical
viewpoint: the availability of a similarity metric. \citet{DHPR+12} are well
aware of this issue; they write that justifying this assumption is ``one of the
most challenging aspects'' of their approach. They add that ``in reality the
metric used will most likely only be society's current best approximation to the
truth.'' But, despite recent progress on automating ethical decisions in certain
domains~\citep{NGAD+18,FSSD+18}, the task-specific nature of the similarity
metric makes even a credible approximation thereof seem unrealistic. In
particular, if one wanted to learn a similarity metric, it is unclear what type
of examples a relevant dataset would consist of.

In place of a metric, envy-freeness requires access to individuals' utility
functions, but --- by contrast --- we do not view this assumption as a
barrier to implementation. Indeed, there are a variety of techniques for
learning utility functions~\citep{CKO01,NJ04,BCIW12}. Moreover, in our running
example of advertising, one can use standard measures like expected
click-through rate (CTR) as a good proxy for utility.

It is worth noting that the classification setting is different from classic
fair division problems in that the ``goods'' (outcomes) are non-excludable. In
fact, one envy-free solution simply assigns each individual to his favorite
outcome. But this solution may be severely suboptimal according to another (standard) component of our setting, the \emph{loss function}, which, in the examples above, might represent the expected revenue from showing an ad to an individual. Typically the loss function is not perfectly aligned with individual utilities, and, therefore, it may
be possible to achieve smaller loss than the na\"ive solution without violating the envy-freeness
constraint.

In summary, we view envy-freeness as a compelling, well-established, and,
importantly, practicable notion of individual fairness for classification tasks
with a diverse set of outcomes when individuals have heterogeneous preferences.
Our goal is to understand its learning-theoretic properties.

\vspace{-0.1in}

\subsection{Our Results}

The challenge is that the space of individuals is potentially huge, yet we seek to provide universal envy-freeness guarantees. To this end, we are given a sample consisting of individuals drawn from an unknown distribution. We are interested in learning algorithms that minimize loss, subject to satisfying the envy-freeness constraint, \emph{on the sample}. Our primary technical question is that of generalizability, that is, \emph{given a classifier that is envy free on a sample, is it approximately envy free on the underlying distribution?} Surprisingly, \citet{DHPR+12} do not study generalizability in their model, and we are aware of only one subsequent paper that takes a learning-theoretic viewpoint on individual fairness and gives theoretical guarantees (see Section~\ref{subsec:rel}). 

In Section~\ref{sec:arb}, we do not constrain the classifier. Therefore, we need some strategy to extend a classifier that is defined on a sample; assigning an individual the same outcome as his \emph{nearest neighbor} in the sample is a popular choice. However, we show that \emph{any} strategy for extending a classifier from a sample, on which it is envy free, to the entire set of individuals is unlikely to be approximately envy free on the distribution, unless the sample is exponentially large.

For this reason, in Section~\ref{sec:mixtures}, we focus on structured families of classifiers. On a high level, our goal is to relate the combinatorial richness of the family to generalization guarantees. One obstacle is that standard notions of dimension do not extend to the analysis of randomized classifiers, whose range is \emph{distributions} over outcomes (equivalently, real vectors). We circumvent this obstacle by considering mixtures of \emph{deterministic} classifiers that belong to a family of bounded Natarajan dimension (an extension of the well-known VC dimension to multi-class classification). Our main theoretical result asserts that, under this assumption, envy-freeness on a sample does generalize to the underlying distribution, even if the sample is relatively small (its size grows almost linearly in the Natarajan dimension).

Finally, in Section~\ref{sec:expts}, we design and implement an algorithm that learns (almost) envy-free mixtures of linear one-vs-all classifiers. We present empirical results that validate our computational approach, and indicate good generalization properties even when the sample size is small.

\subsection{Related Work}
\label{subsec:rel}

Conceptually, our work is most closely related to work by \citet{ZVGG+17}. They are interested in group notions of fairness, and advocate preference-based notions instead of parity-based notions. In particular, they assume that each group has a utility function for \emph{classifiers}, and define the \emph{preferred treatment} property, which requires that the utility of each group for its own classifier be at least its utility for the classifier assigned to any other group. Their model and results focus on the case of binary classification where there is a desirable outcome and an undesirable outcome, so the utility of a group for a classifier is simply the fraction of its members that are mapped to the desirable outcome. Although, at first glance, this notion seems similar to envy-freeness, it is actually fundamentally different.\footnote{On a philosophical level, the fair division literature deals exclusively with individual notions of fairness. In fact, even in group-based extensions of envy-freeness~\citep{MS17} the allocation is shared by groups, but individuals must not be envious. We subscribe to the view that group-oriented notions (such as statistical parity) are objectionable, because the outcome can be patently unfair to individuals.}
Our paper is also completely different from that of \citeauthor{ZVGG+17}~in terms of technical results; theirs are purely empirical in nature, and focus on the increase in accuracy obtained when parity-based notions of fairness are replaced with preference-based ones.

Concurrent work by \citet{RY18} provides generalization guarantees for the metric notion of individual fairness introduced by \citet{DHPR+12}, or, more precisely, for an approximate version thereof. 
There are two main differences compared to our work: first, we propose envy-freeness as an alternative notion of fairness that circumvents the need for a similarity metric. Second, they focus on randomized \emph{binary} classification, which amounts to learning a real-valued function, and so are able to make use of standard Rademacher complexity results to show generalization. By contrast, standard tools do not directly apply in our setting. It is worth noting that several other papers provide generalization guarantees for notions of group fairness, but these are more distantly related to our work \citep{ZWSP+13,WGOS17,DOBS+18,KNRW18,HKRR18}.

\section{The Model} \label{sec:problem}

We assume that there is
a space $\cX$ of individuals, a finite space $\cY$ of outcomes, and a utility function $u : \cX \times \cY \to [0,1]$ encoding the preferences
of each individual for the outcomes in $\cY$. In the advertising example, individuals are users, outcomes are advertisements, and the utility function reflects the benefit an individual derives from being shown a particular advertisement. For any distribution $p \in \Delta(\cY)$ (where $\Delta(\cY)$ is the set of distributions over $\cY$) we let
$u(x,p) = \expect_{y \sim p}[u(x,y)]$ denote individual $x$'s expected utility
for an outcome sampled from $p$. We refer to a function $h :
\cX \to \Delta(\cY)$ as a \emph{classifier}, even though it can return a distribution over outcomes.

\subsection{Envy-Freeness}

Roughly speaking, a classifier $h :
\cX \to \Delta(\cY)$ is envy free if no individual prefers the
outcome distribution of someone else over his own.

\begin{defn}
	A classifier $h : \cX \to \Delta(\cY)$ is \emph{envy free (EF)} on a set $S$ of
	individuals if $u(x,h(x)) \geq u(x,h(x'))$ for all $x,x' \in S$. Similarly,
	$h$ is \emph{$(\alpha,\beta)$-EF} with respect to a distribution $P$ on $\cX$ if
	\[\prob_{x,x'\sim P}\bigl( u(x,h(x)) < u(x,h(x')) - \beta \bigr) \leq \alpha.\]
	Finally, $h$ is \emph{$(\alpha,\beta)$-pairwise EF} on a set of pairs of
	individuals $S = \{(x_i,x'_i)\}_{i=1}^n$ if \[\frac{1}{n} \sum_{i=1}^n
	\ind{u(x_i, h(x_i)) < u(x_i, h(x'_i)) - \beta} \leq \alpha.\]
\end{defn}

Any classifier that is EF on a sample $S$ of individuals is also
$(\alpha,\beta)$-pairwise EF on any pairing of the individuals in $S$, for any $\alpha \geq 0$ and $\beta \geq 0$. The
weaker pairwise EF condition is all that is required for our
generalization guarantees to hold.

\subsection{Optimization and Learning}

Our formal learning problem can be stated as follows. Given sample access to an
unknown distribution $P$ over individuals $\cX$ and their utility functions, and a known loss function $\ell : \cX \times \cY \to
[0,1]$, find a classifier $h : \cX \to \Delta(\cY)$ that is
$(\alpha,\beta)$-EF with respect to $P$ minimizing expected loss $\expect_{x \sim
	P}[\ell(x, h(x))]$, where for $x \in \cX$ and $p \in \Delta(\cY)$, $\ell(x,p) =
\expect_{y \sim p}[\ell(x,y)]$.

We follow the empirical risk minimization (ERM)
learning approach, i.e., we collect a sample of individuals drawn i.i.d~from
$P$ and find an EF classifier with low loss on the sample. Formally, given a sample of individuals $S=\{x_1,\ldots,x_n\}$ and their utility functions $u_{x_i}(\cdot)=u(x_i,\cdot)$, we are interested in a classifier $h:S\to\Delta(\cY)$ that minimizes $\sum_{i=1}^n \ell(x_i,h(x_i))$ among all classifiers that are EF on $S$.


Recall that we consider randomized classifiers that can assign a distribution over outcomes to each of the individuals. However, one might wonder whether the EF classifier that minimizes loss on a sample happens to always be deterministic. Or, at least, the optimal deterministic classifier on the sample might incur a loss that is very close to that of the optimal randomized classifier. If this were true, we could restrict ourselves to classifiers of the form $h: \cX \to \cY$, which would be much easier to analyze. Unfortunately, it turns out that this is not the case. In fact, there could be an arbitrary (multiplicative) gap between the optimal randomized EF classifier and the optimal deterministic EF classifier. The intuition behind this is as follows. A deterministic classifier that has very low loss on the sample, but is not EF, would be completely discarded in the deterministic setting. On the other hand, a randomized classifier could take this loss-minimizing deterministic classifier and mix it with a classifier with high ``negative envy'', so that the mixture ends up being EF and at the same time has low loss. This is made concrete in the following example.

\begin{example}	\label{ex:det-random-gap}
	Let $S = \{x_1, x_2\}$ and $\cY = \{y_1, y_2, y_3\}$. Let the loss function be such that
	\begin{align*}
	&\ell(x_1, y_1) = 0 \qquad \ell(x_1, y_2) = 1 \qquad \ell(x_1, y_3) = 1\\
	&\ell(x_2, y_1) = 1 \qquad \ell(x_2, y_2) = 1 \qquad \ell(x_2, y_3) = 0
	\end{align*}
	Moreover, let the utility function be such that
	\begin{align*}
	&u(x_1, y_1) = 0 \qquad u(x_1, y_2) = 1 \qquad u(x_1, y_3) = \frac{1}{\gamma}\\
	&u(x_2, y_1) = 0 \qquad u(x_2, y_2) = 0 \qquad u(x_2, y_3) = 1
	\end{align*}
	where $\gamma > 1$. The only deterministic classifier with a loss of $0$ is $h_0$ such that $h_0(x_1) = y_1$ and $h_0(x_2) = y_3$. But, this is not EF, since $u(x_1, y_1) < u(x_1, y_3)$. Furthermore, every other deterministic classifier has a total loss of at least $1$, causing the optimal deterministic EF classifier to have loss of at least $1$.

	To show that randomized classifiers can do much better, consider the randomized classifier $h_*$ such that $h_*(x_1) = \left(1-1/\gamma, 1/\gamma, 0\right)$ and $h_*(x_2) = \left(0, 0, 1\right)$. This classifier can be seen as a mixture of the classifier $h_0$ of $0$ loss, and the deterministic classifier $h_e$, where $h_e(x_1) = y_2$ and $h_e(x_2) = y_3$, which has high ``negative envy". One can observe that this classifier $h_*$ is EF, and has a loss of just $1/\gamma$. Hence, the loss of the optimal randomized EF classifier is $\gamma$ times smaller than the loss of the optimal deterministic one, for any $\gamma > 1$.
\end{example}

%
%

\section{Arbitrary Classifiers}\label{sec:arb}

An important (and typical) aspect of our learning problem is that the classifier
$h$ needs to provide an outcome distribution for every individual, not just
those in the sample. For example, if $h$ chooses advertisements for visitors of
a website, the classifier should still apply when a new visitor arrives.
Moreover, when we use the classifier for new individuals, it must continue to be
EF. In this section, we consider two-stage approaches that first choose outcome
distributions for the individuals in the sample, and then extend those decisions
to the rest of $\cX$.

In more detail, we are given a sample $S = \{x_1, \dots, x_n\}$ of individuals
and a classifier $h : S \to \Delta(\cY)$ assigning outcome distributions to each
individual. Our goal is to extend these assignments to a classifier
$\overline{h} : \cX \to \Delta(\cY)$ that can be applied to new individuals as
well. For example, $h$ could be the loss-minimizing EF classifier on the sample
$S$.

For this section, we
assume that $\cX$ is equipped with a distance metric $d$. Moreover, we assume in this section that the utility function $u$ is $L$-Lipschitz on $\cX$. That is, for
every $y \in \cY$ and for all $x, x' \in \cX$, we have  $|u(x, y) - u(x', y)|
\leq L \cdot d(x,x')$.

Under the foregoing assumptions, one natural way to extend the classifier on the sample
to all of $\cX$ is to assign new individuals the same outcome distribution
as their nearest neighbor in the sample. Formally, for a set $S \subset \cX$ and any individual
$x \in \cX$, let $\NN_S(x)\in \text{arg\,min}_{x'\in S}d(x,x')$ denote the nearest neighbor of $x$ in $S$ with
respect to the metric $d$ (breaking ties arbitrarily). The following simple result (whose proof is relegated to Appendix~\ref{app:arb}) establishes that this approach preserves envy-freeness in cases where
the sample is exponentially large.

\begin{thm} \label{thm:nn-upper-bound}
	Let $d$ be a metric on $\cX$, $P$ be a distribution on $\cX$, and $u$ be
	an $L$-Lipschitz utility function. Let $S$ be a set of individuals such that
	there exists $\hat \cX \subset \cX$ with $P(\hat \cX) \geq 1-\alpha$ and
	$\sup_{x \in \hat \cX} d(x, \NN_S(x)) \leq \beta/(2L)$. Then for any
	classifier $h : S \to \Delta(\cY)$ that is EF  on $S$, the extension
	$\overline{h} : \cX \to \Delta(\cY)$ given by $\overline{h}(x) = h(\NN_S(x))$
	is $(\alpha, \beta)$-EF on $P$.
\end{thm}

The conditions of Theorem~\ref{thm:nn-upper-bound} require that the set of
individuals $S$ is a $\beta/(2L)$-net for at least a $(1-\alpha)$-fraction of
the mass of $P$ on $\cX$. In several natural situations, an exponentially large
sample guarantees that this occurs with high probability. For example, if $\cX$
is a subset of $\reals^q$, $d(x,x') = \norm{x-x'}_2$, and $\cX$ has diameter at
most $D$, then for any distribution $P$ on $\cX$, if $S$ is an i.i.d.~sample of size
$O(\frac{1}{\alpha}(\frac{LD\sqrt{q}}{\beta})^q(q \log \frac{LD\sqrt{q}}{\beta} +
\log \frac{1}{\delta}))$, it will satisfy the conditions of
Theorem~\ref{thm:nn-upper-bound} with probability at least $1-\delta$. This
sampling result is folklore, but, for the sake of completeness, we prove it in Lemma~\ref{lem:euclideanCover}
of Appendix~\ref{app:arb}.

However, the exponential upper bound given by the nearest neighbor strategy is as far as we can go in terms of generalizing envy-freeness from a sample (without further assumptions). Specifically, our next result establishes that \emph{any} algorithm --- even randomized --- for extending classifiers from the sample to the entire space $\cX$
requires an exponentially large sample of individuals to ensure envy-freeness on
the distribution $P$. The proof of Theorem~\ref{thm:lb} can be found in Appendix~\ref{app:arb}.

\begin{thm}
	\label{thm:lb}
	There exists a space of individuals $\cX \subset \reals^q$, and a distribution $P$ over $\cX$ such that, for every randomized algorithm $\mathcal{A}$ that extends classifiers on a sample to $\cX$, there exists an $L$-Lipschitz utility function $u$ such that, when a sample of individuals $S$ of size $n = 4^q / 2$ is drawn from $P$ without replacement, there exists an EF classifier on $S$ for which, with probability at least $1 - 2\exp(-4^q/100) - \exp(-4^q/200)$ jointly over the randomness of $\mathcal{A}$ and $S$, its extension by $\mathcal{A}$ is not $(\alpha, \beta)$-EF with respect to $P$ for any $\alpha < 1/25$ and $\beta < L/8$.
\end{thm}

We remark that a similar result would hold even if we sampled $S$ with replacement; we sample here without replacement purely for ease of exposition.

\section{Low-Complexity Families of Classifiers} \label{sec:mixtures}

In this section we show that (despite Theorem~\ref{thm:lb}) generalization for
envy-freeness is possible using much smaller samples of individuals, as long as
we restrict ourselves to classifiers from a family of relatively low complexity.

In more detail, two classic complexity measures are the VC-dimension~\citep{VC}
for binary classifiers, and the Natarajan
dimension~\citep{Natarajan89:Dimension} for multi-class classifiers. However, to
the best of our knowledge, there is no suitable dimension directly applicable to
functions ranging over distributions, which in our case can be seen as
$|\cY|$-dimensional real vectors. One possibility would be to restrict ourselves
to deterministic classifiers of the type $h : \cX \to \cY$, but we have
seen in Section~\ref{sec:problem} that envy-freeness is a very strong constraint
on deterministic classifiers. Instead, we will consider a family $\cH$
consisting of randomized mixtures of $m$ deterministic classifiers belonging to a
family $\cG \subset \{g : \cX \to \cY\}$ of low Natarajan dimension. This allows
us to adapt Natarajan-dimension-based generalization results to our setting
while still working with randomized classifiers. The definition and relevant properties of the Natarajan dimension are summarized in Appendix~\ref{app:natar}.

Formally,
let $\vec{g} = (g_1, \dots, g_m) \in
\cG^m$ be a vector of $m$ functions in $\cG$ and $\weights \in \Delta_m$ be a
distribution over $[m]$, where $\Delta_m = \{ p \in \reals^m \,:\, p_i \geq 0,
\sum_i p_i = 1\}$ is the $m$-dimensional probability simplex. Then consider the
function $h_{\vec{g},\weights} : \cX \to \Delta(\cY)$ with assignment
probabilities given by
$
\prob(h_{\vec{g}, \weights}(x) = y)
=
\sum_{i=1}^m \ind{g_i(x) = y} \weights_i.
$
Intuitively, for a given individual $x$, $h_{\vec{g}, \weights}$ chooses one of
the $g_i$ randomly with probability $\weights_i$, and outputs $g_i(x)$. Let
\[\cH(\cG, m) = \{ h_{\vec{g},\weights} : \cX \to \Delta(\cY) \,:\, \vec{g} \in
\cG^m, \weights \in \Delta_m\}\] be the family of classifiers that can be written
this way. Our main technical result shows that envy-freeness generalizes for
this class.

\begin{thm} \label{thm:mixtureGeneralize}
	Suppose $\cG$ is a family of deterministic classifiers of Natarajan
	dimension $d$, and let $\cH = \cH(\cG,m)$ for $m\in \mathbb{N}$. For any distribution
	$P$ over $\cX$, $\gamma>0$, and $\delta > 0$, if $S = \{(x_i, x'_i)\}_{i=1}^n$
	is an i.i.d.~sample of pairs drawn from $P$ of size \[n \geq
	O\left(\frac{1}{\gamma^2}\left(dm^2 \log \frac{dm|\cY|\log(m|\cY|/\gamma)}{\gamma} + \log
	\frac{1}{\gamma}\right)\right),\] then with probability at least $1-\delta$, every classifier $h
	\in \cH$ that is $(\alpha,\beta)$-pairwise-EF on $S$ is also
	$(\alpha+7\gamma, \beta+4\gamma)$-EF on $P$.
\end{thm}

The proof of Theorem~\ref{thm:mixtureGeneralize} is relegated to Appendix~\ref{app:mixtures}. In a nutshell, it consists of two steps. First,
we show that envy-freeness generalizes for finite classes. Second, we show that
$\cH(\cG,m)$ can be approximated by a finite subset.

We remark that the theorem is only effective insofar as families of classifiers of low Natarajan dimension are useful. Fortunately, several prominent families indeed have low Natarajan dimension~\citep{DSS12}, including one vs.~all, multiclass SVM, tree-based classifiers, and error correcting output codes.

\section{Implementation and Empirical Validation}
\label{sec:expts}

So far we have not directly addressed the problem of \emph{computing} the loss-minimizing envy-free classifier from a given family on a given sample of individuals. We now turn to this problem. Our goal is not to provide an end-all solution, but rather to provide evidence that computation will not be a long-term obstacle to implementing our approach.

In more detail, our computational problem is to find the loss-minimizing classifier $h$ from a given family of randomized classifiers $\cH$ that is envy free on a given a sample of individuals $S = \{x_1, \dots, x_n\}$. For this classifier $h$ to generalize to the distribution $P$, Theorem~\ref{thm:mixtureGeneralize} suggests that the family $\cH$ to use is of the form $\cH(\cG, m)$, where $\cG$ is a family of deterministic classifiers of low Natarajan dimension.

In this section, we let $\cG$ be the family of \emph{linear one-vs-all classifiers}. In particular, denoting $\cX \subset \mathbb{R}^q$, each $g \in \cG$ is parameterized by $\vec{w} = (w_1, w_2, \dots, w_{|\cY|}) \in \mathbb{R}^{|\cY| \times q}$, where
$g(x) = \text{argmax}_{y \in \cY} \left(w_y^\tp x\right)$.
This class $\cG$ has a Natarajan dimension of at most $q|\cY|$. The optimization problem to solve in this case is
\begin{align} \label{eqn:opt-orig}
&\min_{\vec{g} \in \cG^m, \weights \in \Delta_m} \quad \sum_{i=1}^n \sum_{k=1}^m \weights_k L(x_i, g_k(x_i))\notag\\
&\text{s.t.} \quad \sum_{k=1}^m \weights_k u(x_i, g_k(x_i)) \geq \sum_{k=1}^m \weights_k u(x_i, g_k(x_j)) \quad \forall (i,j) \in [n]^2.
\end{align}


\subsection{Algorithm}	\label{subsec:algo}

Observe that optimization problem~\eqref{eqn:opt-orig} is highly non-convex and non-differentiable as formulated, because of the $\text{argmax}$ computed in each of the $g_k(x_i)$. Another challenge is the combinatorial nature of the problem, as we need to find $m$ functions from $\cG$ along with their mixing weights. In designing an algorithm, therefore, we employ several tricks of the trade to achieve tractability.

\noindent\textbf{Learning the mixture components.}
We first assume predefined mixing weights $\tilde{\weights}$, and \emph{iteratively} learn mixture components based on them. Specifically, let $g_1, g_2, \dots g_{k-1}$ denote the classifiers learned so far. To compute the next component $g_k$, we solve the optimization problem~\eqref{eqn:opt-orig} with these components already in place (and assuming no future ones). This induces the following optimization problem.
\begin{align}	\label{eqn:greedy-opt}
&\min_{g_k \in \cG} \quad \sum_{i=1}^n L(x_i, g_k(x_i))\notag\\
&\text{s.t.} \quad USF_{ii}^{(k-1)} + \tilde{\weights}_k u(x_i, g_k(x_i))\geq USF_{ij}^{(k-1)} + \tilde{\weights}_k u(x_i, g_k(x_j)) \quad \forall (i,j) \in [n]^2,
\end{align}
where $USF_{ij}^{(k-1)}$ denotes the expected utility $i$ has for $j$'s assignments so far, i.e., $USF_{ij}^{(k-1)} = \sum_{c=1}^{k-1} \tilde{\weights}_c u(x_i, g_c(x_j))$.

Solving the optimization problem~\eqref{eqn:greedy-opt} is still non-trivial
because it remains non-convex and non-differentiable. To resolve this, we first soften the constraints\footnote{This may lead to solutions that are not exactly EF on the sample. Nonetheless, Theorem~\ref{thm:mixtureGeneralize} still guarantees that there should not be much additional envy on the testing data.}. Writing out the optimization problem in the form equivalent to introducing slack variables, we obtain
\begin{align}	\label{eqn:softened-opt-main}
&\min_{g_k \in \cG} \quad \sum_{i=1}^n L(x_i, g_k(x_i)) \notag\\
&\qquad\quad+ \lambda \sum_{i \neq j} \max\left(USF_{ij}^{(k-1)} + \tilde{\weights}_k u(x_i, g_k(x_j)) - USF_{ii}^{(k-1)} - \tilde{\weights}_k u(x_i, g_k(x_i)), 0\right),
\end{align}
where $\lambda$ is a parameter that defines the trade-off between loss and envy-freeness.
This optimization problem is still highly non-convex as $g_k(x_i) = \text{argmax}_{y \in \cY} w_y^\tp x_i$, where $\vec{w}$ denotes the parameters of $g_k$. To solve this, we perform a convex relaxation on several components of the objective using the fact that $w_{g_k(x_i)}^\tp x_i \geq w_{y'}^\tp x_i$ for any $y' \in \cY$. Specifically, we have
$$L(x_i, g_k(x_i)) \leq \max_{y \in \cY}\left\{L(x_i, y) + w_{y}^\tp x_i - w_{y_i}^\tp x_i\right\},$$
$$-u(x_i, g_k(x_i)) \leq \max_{y \in \cY}\left\{-u(x_i, y) + w_y^\tp x_i - w_{b_i}^\tp x_i\right\}, \text{ and}$$
$$u(x_i, g_k(x_j)) \leq \max_{y \in \cY}\left\{u(x_i, y) + w_y^\tp x_j - w_{s_i}^\tp x_j\right\},$$
where $y_i = \text{argmin}_{y \in \cY} L(x_i, y)$, $s_i = \text{argmin}_{y \in \cY} u(x_i, y)$ and $b_i = \text{argmax}_{y \in \cY} \ u(x_i, y)$.
While we provided the key steps here, complete details and the rationale behind these choices are given in Appendix~\ref{app:expts}. On a very high-level, these are inspired by multi-class SVMs. Finally, plugging these relaxations into~\eqref{eqn:softened-opt-main}, we obtain the following convex optimization problem to compute each mixture component.
\begin{align}	\label{eqn:relaxed-opt}
&\min_{
  \vec{w} \in
  \mathbb{R}^{|\cY|\times q}
}
\quad \sum_{i=1}^n \max_{y \in \cY} \left\{L(x_i, y) +  w_y^\tp x_i - w_{y_i}^\tp x_i\right\}+ \lambda \sum_{i \neq j} \max\left(USF_{ij}^{(k-1)}\right.\\
&\left.+ \tilde{\weights}_k \max_{y \in \cY}\left\{u(x_i, y) + w_y^\tp x_j - w_{s_i}^\tp x_j\right\}- USF_{ii}^{(k-1)} + \tilde{\weights}_k \max_{y \in \cY}\left\{-u(x_i, y) + w_y^\tp x_i - w_{b_i}^\tp x_i\right\}, 0\right).\notag
\end{align}

\noindent\textbf{Learning the mixing weights.}
Once the mixture components $\vec{g}$ are learned (with respect to the predefined mixing weights $\tilde{\weights}$), we perform an additional round of optimization to learn the optimal weights $\weights$ for them. This can be done via the following linear program
\begin{align} \label{eqn:alpha-opt}
&\min_{\weights \in \Delta_m, \xi \in \mathbb{R}_{\geq 0}^{n \times n}} \quad \sum_{i=1}^n \sum_{k=1}^m \weights_k L(x_i, g_k(x_i)) + \lambda \sum_{i \neq j} \xi_{ij}\notag\\
&\text{s.t.} \quad \sum_{k=1}^m \weights_k u(x_i, g_k(x_i)) \geq \sum_{k=1}^m \weights_k u(x_i, g_k(x_j)) - \xi_{ij} \quad \forall (i,j).
\end{align}

\subsection{Methodology}	\label{subsec:data-gen}

To validate our approach, we have implemented our algorithm. However, we cannot
rely on standard datasets, as we need access to both the features and the
utility functions of individuals. Hence, we rely on synthetic data. All our code is included as supplementary material.
Our experiments are
carried out on a desktop machine with 16GB memory and an Intel Xeon(R) CPU
E5-1603 v3 @ 2.80GHz$\times$4 processor. To solve convex optimization
problems, we use CVXPY \citep{cvxpy, cvxpy_rewriting}.

In our experiments, we cannot compute the optimal solution to the original optimization problem~\eqref{eqn:opt-orig}, and there are no existing methods we can use as benchmarks. Hence, we generate the dataset such that we know the optimal solution upfront.

Specifically, to generate the whole dataset (both training and test), we first generate random classifiers $\vec{g}^\star \in \cG^m$ by sampling their parameters $\vec{w}_1, \dots \vec{w}_m \sim \mathcal{N}(0,1)^{|\cY|\times q}$, and generate $\weights^\star \in \Delta_m$ by drawing uniformly random weights in $[0,1]$ and normalizing. We use $h_{\vec{g}^\star, \weights^\star}$ as the optimal solution of the dataset we generate. For each individual, we sample each feature value independently and u.a.r.~in $[0,1]$.
For each individual $x$ and outcome $y$, we set $L(x,y) = 0$ if $y \in \{g_k^\star(x) : k \in [m]\}$ and otherwise we sample $L(x,y)$ u.a.r. in $[0,1]$.
For the utility function $u$, we need to generate it such that the randomized classifier $h_{\vec{g}^\star, \weights^\star}$ is envy free on the dataset. For this, we set up a linear program and compute each of the values $u(x,y)$. Hence, $h_{\vec{g}^\star, \weights^\star}$ is envy free and has zero loss, so it is obviously the optimal solution. The dataset is split into 75\% training data (to measure the accuracy of our solution to the optimization problem) and 25\% test data (to evaluate generalizability).

For our experiments, we use the following parameters: $|\cY| = 10$, $q = 10$, $m
= 5$, and $\lambda = 10.0$. We set the predefined weights to be
$\tilde{\weights} = \left[\frac{1}{2}, \frac{1}{4}, \dots, \frac{1}{2^{m-1}},
\frac{1}{2^{m-1}}\right]$.\footnote{\textcolor{black}{The reason for using an exponential decay is
so that the subsequent classifiers learned are different from the previous ones.
Using smaller weights
might cause consecutive classifiers to be identical, thereby
`wasting' some of the components.
}} In our experiments we vary the number of individuals, and
each result is averaged over 25 runs. On each run, we generate a
new ground-truth classifier $h_{\vec{g}^*, \weights^*}$, as well as new
individuals, losses, and utilities.

\subsection{Results}

Figure~\ref{fig:times} shows the time taken to compute the mixture components $\vec{g}$ and the optimal weights $\weights$, as the number of individuals in the training data increases. As we will see shortly, even though the $\weights$ computation takes a very small fraction of the time, it can lead to non-negligible gains in terms of loss and envy.

Figure~\ref{fig:losses} shows the average loss attained on the training and test
data by the algorithm immediately after computing the mixture components, and
after the round of $\weights$ optimization. It also shows the average loss
attained (on both the training and test data) by a random allocation, which
serves as a na\"ive benchmark for calibration purposes. Recall that the optimal
assignment $h_{\vec{g}^\star, \weights^\star}$ has loss $0$.
\textcolor{black}{For both the training and testing individuals, optimizing
$\weights$ improves the loss of the learned classifer. Moreover, our algorithms
achieve low training errors for all dataset sizes, and as the dataset grows the
testing error converges to the training error.}

Figure~\ref{fig:envies} shows the average envy among pairs in the training data
and test data, where, for each pair, negative envy is replaced with $0$, to
avoid obfuscating positive envy. The graph also depicts the average envy
attained (on both the training and test data) by a random allocation.
As for the losses, optimizing $\weights$ results in lower
average envy, and as the training set grows we see the generalization gap
decrease.

In Figure~\ref{fig:envyCDF} we zoom in on the case of $100$ training individuals, and observe the empirical CDF of envy values. Interestingly, the optimal randomized classifier $h_{\vec{g}^\star, \weights^\star}$ shows lower negative envy values compared to other algorithms, but as expected has no positive envy pairs. Looking at the positive envy values, we can again see very encouraging results. In particular, for at least a $0.946$ fraction of the pairs in the train data, we obtain envy of at most $0.05$, and this generalizes to the test data, where for at least a $0.939$ fraction of the pairs, we obtain envy of at most $0.1$.

In summary, these results indicate that the algorithm described in Section~\ref{subsec:algo} solves the optimization problem~\eqref{eqn:opt-orig} for linear one-vs-all classifiers almost optimally, and that its output generalizes well even when the training set is small.

\section{Conclusion}\label{sec:conc}

In this paper we propose EF as a suitable fairness notion for learning tasks with
many outcomes over which individuals have heterogeneous preferences. We provide
generalization guarantees for a rich family of classifiers, showing that if we
find a classifier that is envy-free on a sample of individuals, it will remain
envy-free when we apply it to new individuals from the same distribution. This
result circumvents an exponential lower bound on the sample complexity suffered
by any two-stage learning algorithm that first finds an EF assignment for the
sample and then extends it to the entire space. Finally, we empirically
demonstrate that finding low-envy and low-loss classifiers is computationally
tractable. These results show that envy-freeness is a practical notion of
fairness for machine learning systems.

\begin{figure}[t]
    \centering
    \begin{minipage}{0.47\textwidth}
		\centering
		\includegraphics[width=\textwidth]{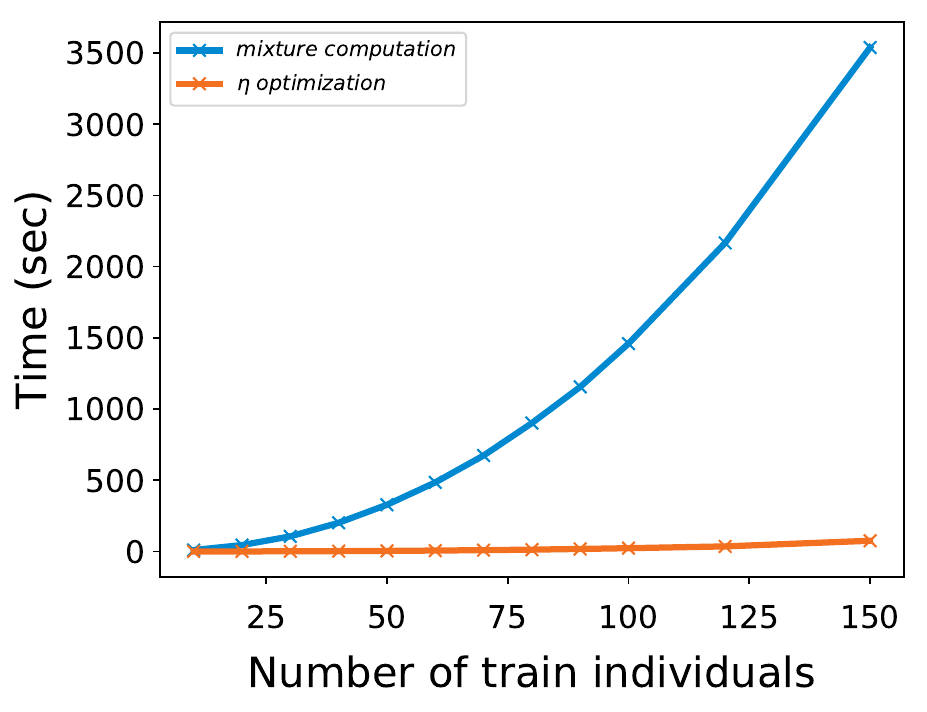}
		\caption{The algorithm's running time.}
		\label{fig:times}
    \end{minipage}\hfill
    \begin{minipage}{0.47\textwidth}
        \centering
		\includegraphics[width=\textwidth]{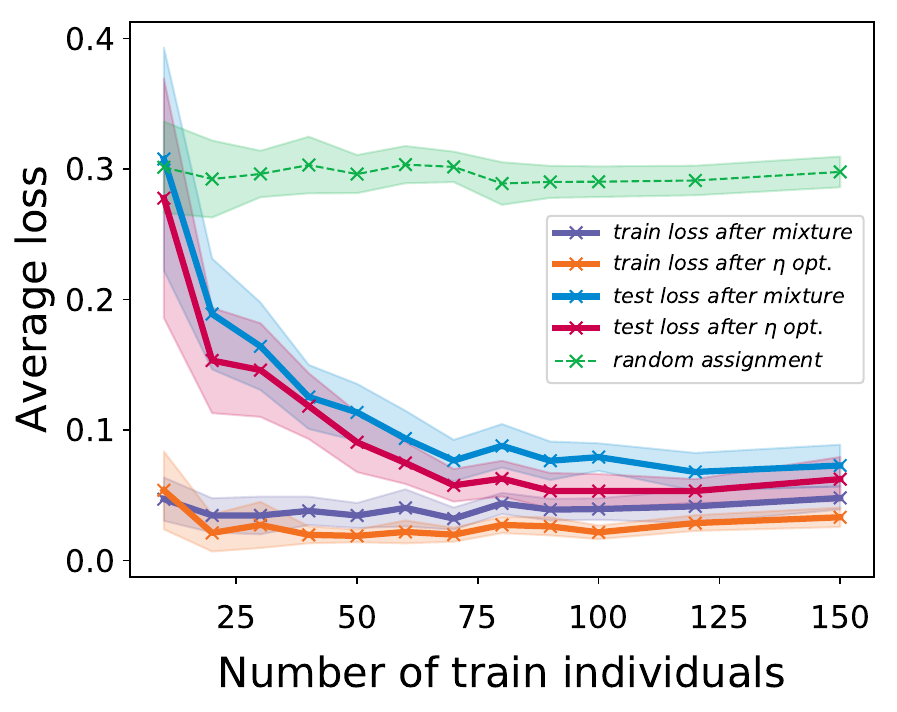}
		\caption{Training and test loss. Shaded error bands depict $95\%$ confidence intervals.
		}
		\label{fig:losses}
    \end{minipage}
\end{figure}

\begin{figure}[t]
    \centering
    \begin{minipage}{0.47\textwidth}
        \centering
        \includegraphics[width=\textwidth]{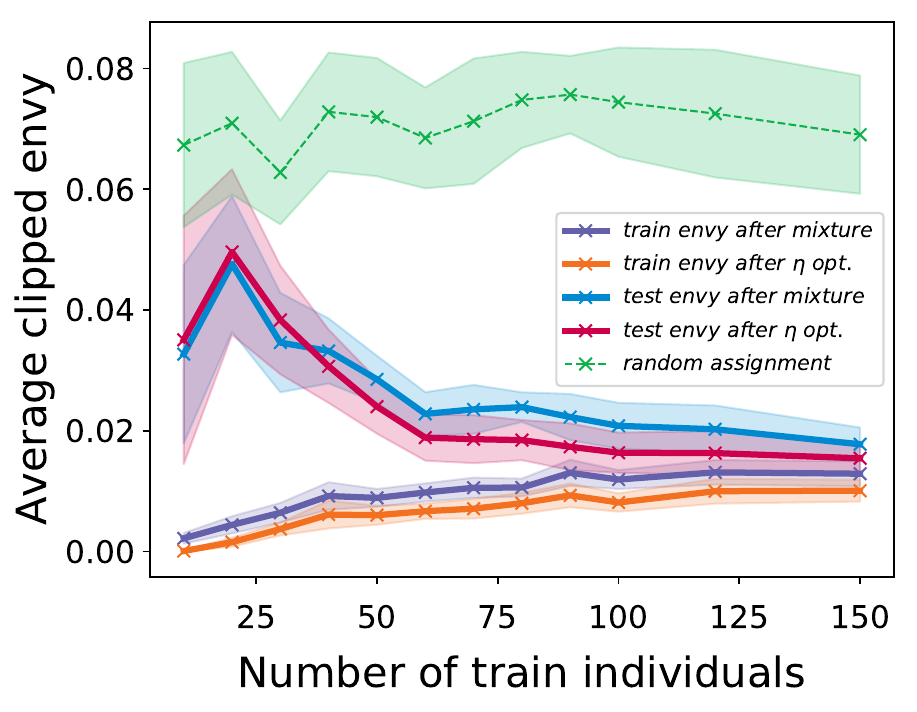}
        \caption{Training and test envy, as a function of the number of individuals.
        Shaded error bands depict $95\%$ confidence intervals.}
		\label{fig:envies}
    \end{minipage}\hfill
    \begin{minipage}{0.47\textwidth}
        \centering
        \includegraphics[width=\textwidth]{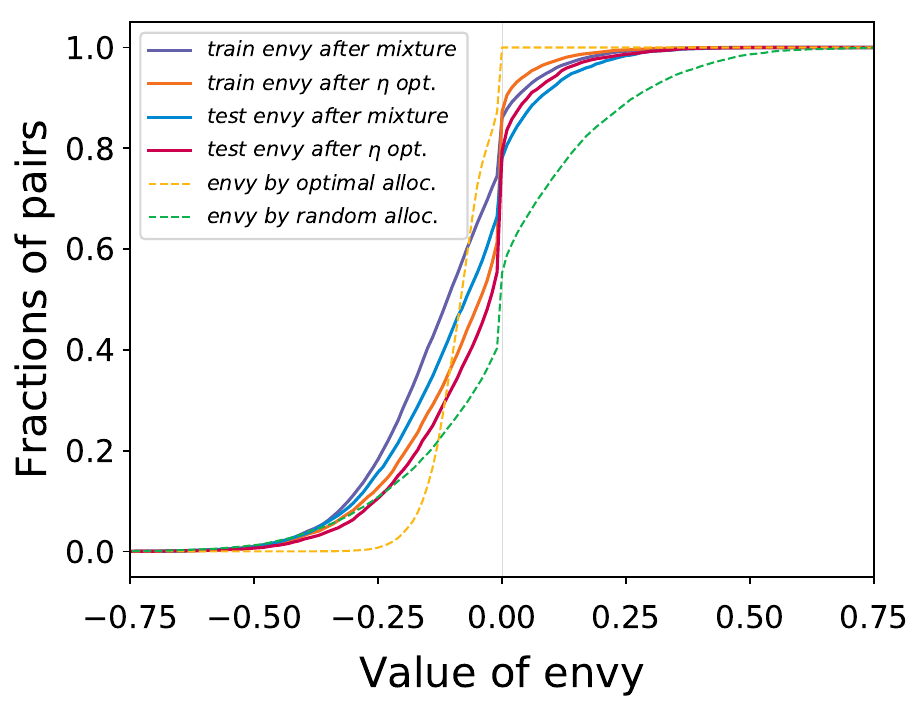}
		\caption{CDF of training and test envy for 100 training individuals}
		\label{fig:envyCDF}
    \end{minipage}
\end{figure}

\section*{Acknowledgments}
This work was partially supported by the National Science Foundation under grants IIS-1350598, IIS-1714140, IIS-1618714, IIS-1901403, CCF-1525932, CCF-1733556, CCF-1535967, CCF-1910321; by the Office of Naval Research under grants N00014-16-1-3075 and N00014-17-1-2428; and by a J.P.~Morgan AI Research Award, an Amazon Research Award, a Microsoft Research Faculty Fellowship, a Bloomberg Data Science research grant, a Guggenheim Fellowship, and a grant from the Block Center for Technology and Society.

\bibliographystyle{plainnat}
\bibliography{abb,ultimate}

\clearpage

\appendix

\begin{center}
{\Large \textbf{Appendix: Envy-Free Classification}}
\end{center}
\bigskip

\section{Natarajan Dimension Primer}
\label{app:natar}

We briefly present the Natarajan dimension. For more details, we refer the reader to
\citep{SBD14:MLBook}.

We say that a family $\cG$ \emph{multi-class shatters} a set of
points $x_1, \dots, x_n$ if there exist labels $y_1, \dots y_n$ and $y'_1,
\dots, y'_n$ such that for every $i \in [n]$ we have $y_i \neq y'_i$, and for any
subset $C \subset[n]$ there exists $g \in \cG$ such that $g(x_i) = y_i$ if $i
\in C$ and $g(x_i)=y'_i$ otherwise. The Natarajan dimension of a family $\cG$ is the
cardinality of the largest set of points that can be multi-class shattered by
$\cG$.

For example, suppose we have a feature map $\Psi : \cX \times \cY \to \reals^q$
that maps each individual-outcome pair to a $q$-dimensional feature vector, and
consider the family of functions that can be written as $g(x) =
\text{arg\,max}_{y \in \cY} w^\tp \Psi(x,y)$ for weight vectors $w \in
\reals^q$. This family has Natarajan dimension at most $q$.

For a set $S \subset \cX$ of points, we let $\cG\resto_S$ denote the restriction
of $\cG$ to $S$, which is any subset of $\cG$ of minimal size such that for
every $g \in \cG$ there exists $g' \in \cG\resto_S$ such that $g(x) = g'(x)$
for all $x \in S$. The size of $\cG\resto_S$ is the number of different
labelings of the sample $S$ achievable by functions in $\cG$. The following
Lemma is the analogue of Sauer's lemma for binary classification.

\begin{lem}[Natarajan] \label{lem:natarajan}
	For a family $\cG$ of Natarajan dimension $d$ and any subset $S \subset \cX$,
	we have $\bigl|\cG\resto_S\bigr| \leq |S|^d |\cY|^{2d}$.
\end{lem}

Classes of low Natarajan dimension also enjoy the following uniform convergence
guarantee.

\begin{lem} \label{lem:natarajanAgnostic}
	Let $\cG$ have Natarajan dimension $d$ and fix a loss function $\ell : \cG
	\times \cX \to [0,1]$. For any distribution $P$ over $\cX$, if $S$ is an i.i.d.~sample drawn from $P$ of size $O(\frac{1}{\epsilon^2}(d \log |\cY| + \log
	\frac{1}{\delta}))$, then with probability at least $1-\delta$ we have
	$
	\sup_{g \in \cG} \, \left|
	\expect_{x \sim P}[\ell(g,x)] - \frac{1}{n}\sum_{x\in S} \ell(g,x)
	\right| \leq \epsilon.
	$
\end{lem}

%
%

\section{Appendix for Section~\ref{sec:arb}}	\label{app:arb}

\noindent\textbf{Theorem~\ref{thm:nn-upper-bound}.}
	\emph{Let $d$ be a metric on $\cX$, $P$ be a distribution on $\cX$, and $u$ be
	an $L$-Lipschitz utility function. Let $S$ be a set of individuals such that
	there exists $\hat \cX \subset \cX$ with $P(\hat \cX) \geq 1-\alpha$ and
	$\sup_{x \in \hat \cX} d(x, \NN_S(x)) \leq \beta/(2L)$. Then for any
	classifier $h : S \to \Delta(\cY)$ that is EF  on $S$, the extension
	$\overline{h} : \cX \to \Delta(\cY)$ given by $\overline{h}(x) = h(\NN_S(x))$
	is $(\alpha, \beta)$-EF on $P$.}

\begin{proof}
	Let $h : S \to \Delta(\cY)$ be any EF classifier on $S$ and
	$\overline h : \cX \to \Delta(\cY)$ be the nearest neighbor extension. Sample $x$ and $x'$ from $P$. Then, $x$ belongs to the subset $\hat \cX$ with probability at least $1-\alpha$. When this occurs, $x$ has a neighbor within distance $\beta/(2L)$ in the sample. Using
	the Lipschitz continuity of $u$, we have $|u(x,\overline{h}(x)) - u(\NN_S(x),
	h(\NN_S(x)))| \leq \beta/2$. Similarly, $|u(x,\overline{h}(x')) - u(\NN_S(x),
	h(\NN_S(x')))| \leq \beta/2$. Finally,
	since $\NN_S(x)$ does not envy $\NN_S(x')$ under $h$, it follows that $x$ does
	not envy $x'$ by more than $\beta$ under $\overline{h}$.
\end{proof}

\medskip

\begin{lem} \label{lem:euclideanCover}
	Suppose $\cX \subset \reals^q$, $d(x,x') = \norm{x-x'}_2$, and let $D =
	\sup_{x,x' \in \cX}d(x,x')$ be the diameter of $\cX$. For any distribution $P$
	over $\cX$, $\beta > 0$, $\alpha > 0$, and $\delta > 0$ there exists $\hat \cX
	\subset \cX$ such that $P(\hat \cX) \geq 1-\alpha$ and, if $S$ is an
	i.i.d.~sample drawn from $P$ of size $|S| =
	O(\frac{1}{\alpha}(\frac{LD\sqrt{q}}{\beta})^q(d \log
	\frac{LD\sqrt{q}}{\beta} + \log \frac{1}{\delta}))$, then with probability at
	least $1-\delta$, $\sup_{x \in \hat \cX} d(x, \NN_S(x)) \leq \beta/(2L)$.
\end{lem}
\begin{proof}
	Let $C$ be the smallest cube containing $\cX$. Since the diameter of $\cX$ is
	$D$, the side-length of $C$ is at most $D$. Let $s = \beta/(2L\sqrt{q})$ be
	the side-length such that a cube with side-length $s$ has diameter
	$\beta/(2L)$. It takes at most $m = \lceil D/s \rceil^q$ cubes of side-length
	$s$ to cover $C$. Let $C_1, \dots, C_m$ be such a covering, where each $C_i$
	has side-length $s$.

	Let $C_i$ be any cube in the cover for which $P(C_i) > \alpha / m$. The
	probability that a sample of size $n$ drawn from $P$ does not contain a sample
	in $C_i$ is at most $(1-\alpha/m)^n \leq e^{-n\alpha/m}$. Let $I = \{i \in [m] \,:\,
	P(C_i) \geq \alpha / m\}$. By the union bound, the probability that there
	exists $i \in I$ such that $C_i$ does not contain a sample is at most
	$me^{-n\alpha/m}$. Setting
	\begin{align*}
	n &= \frac{m}{\alpha}\ln \frac{m}{\delta} \\
	&= O\left(
	\frac{1}{\alpha}
	\biggl(\frac{LD\sqrt{q}}{\beta}\biggr)^q
	\biggl(q \log \frac{LD\sqrt{q}}{\beta} + \log \frac{1}{\delta}\biggr)
	\right)
	\end{align*}
	results in this upper bound being $\delta$. For the remainder of the proof,
	assume this high probability event occurs.

	Now let $\hat \cX = \bigcup_{i \in I} C_i$. For each $j \not \in I$, we know
	that $P(C_j) < \alpha/m$. Since there at most $m$ such cubes, their total
	probability mass is at most $\alpha$. It follows that $P(\hat \cX) \geq
	1-\alpha$. Moreover, every point $x \in \hat \cX$ belongs to one of the cubes
	$C_i$ with $i \in I$, which also contains a sample point. Since the diameter
	of the cubes in our cover is $\beta/(2L)$, it follows that $\dist(x,\NN_S(x))
	\leq \beta/(2L)$ for every $x \in \hat \cX$, as required.
\end{proof}

\noindent\textbf{Theorem~\ref{thm:lb}.}
	\emph{
	There exists a space of individuals $\cX \subset \reals^q$, and a distribution $P$ over $\cX$ such that, for every randomized algorithm $\mathcal{A}$ that extends classifiers on a sample to $\cX$, there exists an $L$-Lipschitz utility function $u$ such that, when a sample of individuals $S$ of size $n = 4^q / 2$ is drawn from $P$ without replacement, there exists an EF classifier on $S$ for which, with probability at least $1 - 2\exp(-4^q/100) - \exp(-4^q/200)$ jointly over the randomness of $\mathcal{A}$ and $S$, its extension by $\mathcal{A}$ is not $(\alpha, \beta)$-EF with respect to $P$ for any $\alpha < 1/25$ and $\beta < L/8$.}

\begin{proof}
		Let the space of individuals be $\cX = [0,1]^q$ and the outcomes be $\cY = \{0,1\}$. We partition the space $\cX$ into cubes of side length $s = 1/4$. So, the total number of cubes is $m = \left(1/s\right)^q = 4^q$. Let these cubes be denoted by $c_1, c_2, \dots c_m$, and let their centers be denoted by $\mu_1, \mu_2, \dots \mu_m$. Next, let $P$ be the uniform distribution over the centers $\mu_1, \mu_2, \dots \mu_m$. For brevity, whenever we say ``utility function'' in the rest of the proof, we mean ``$L$-Lipschitz utility function.''

		To prove the theorem, we use Yao's minimax principle~\citep{Yao77}. Specifically, consider the following two-player zero sum game. Player 1 chooses a deterministic algorithm $\mathcal{D}$ that extends classifiers on a sample to $\cX$, and player 2 chooses a utility function $u$ on $\cX$. For any subset $S \subset \cX$, define the classifier $h_{u,S}:S \to \cY$ by assigning each individual in $S$ to his favorite outcome with respect to the utility function $u$, i.e. $h_{u,S}(x) = \text{arg\,max}_{y \in \cY} u(x,y)$ for each $x \in S$, breaking ties lexicographically. Define the cost of playing algorithm $\mathcal{D}$ against utility function $u$ as the probability over the sample $S$ (of size $m/2$ drawn from $P$ without replacement) that the extension of $h_{u,S}$ by $\mathcal{D}$ is not $(\alpha, \beta)$-EF with respect to $P$ for any $\alpha < 1/25$ and $\beta < L/8$. Yao's minimax principle implies that for any randomized algorithm~$\mathcal{A}$, its expected cost with respect to the worst-case utility function~$u$ is at least as high as the expected cost of any distribution over utility functions that is played against the best deterministic algorithm $\mathcal{D}$ (which is tailored for that distribution). Therefore, we establish the desired lower bound by choosing a specific distribution over utility functions, and showing that the best deterministic algorithm against it has an expected cost of at least $1 - 2\exp(-m/100) - \exp(-m/200)$.

		To define this distribution over utility functions, we first sample outcomes $y_1, y_2, \dots, y_m$ i.i.d. from Bernoulli($1/2$). Then, we associate each cube center $\mu_i$ with the outcome $y_i$, and refer to this outcome as the \textit{favorite} of $\mu_i$. For brevity, let $\neg y$ denote the outcome other than $y$, i.e. $\neg y = (1-y)$. For any $x \in \cX$, we define the utility function as follows. Letting $c_j$ be the cube that $x$ belongs to,
		\begin{equation}
		\label{eq:util}
		u(x,y_j) = L \left[\frac{s}{2} - \|x - \mu_j\|_{\infty}\right]; \quad u(x, \neg y_j) = 0.
		\end{equation}
		See Figure~\ref{fig:grid} for an illustration.

		\begin{figure}[t]
			\centering
			\includegraphics[width=0.4\columnwidth]{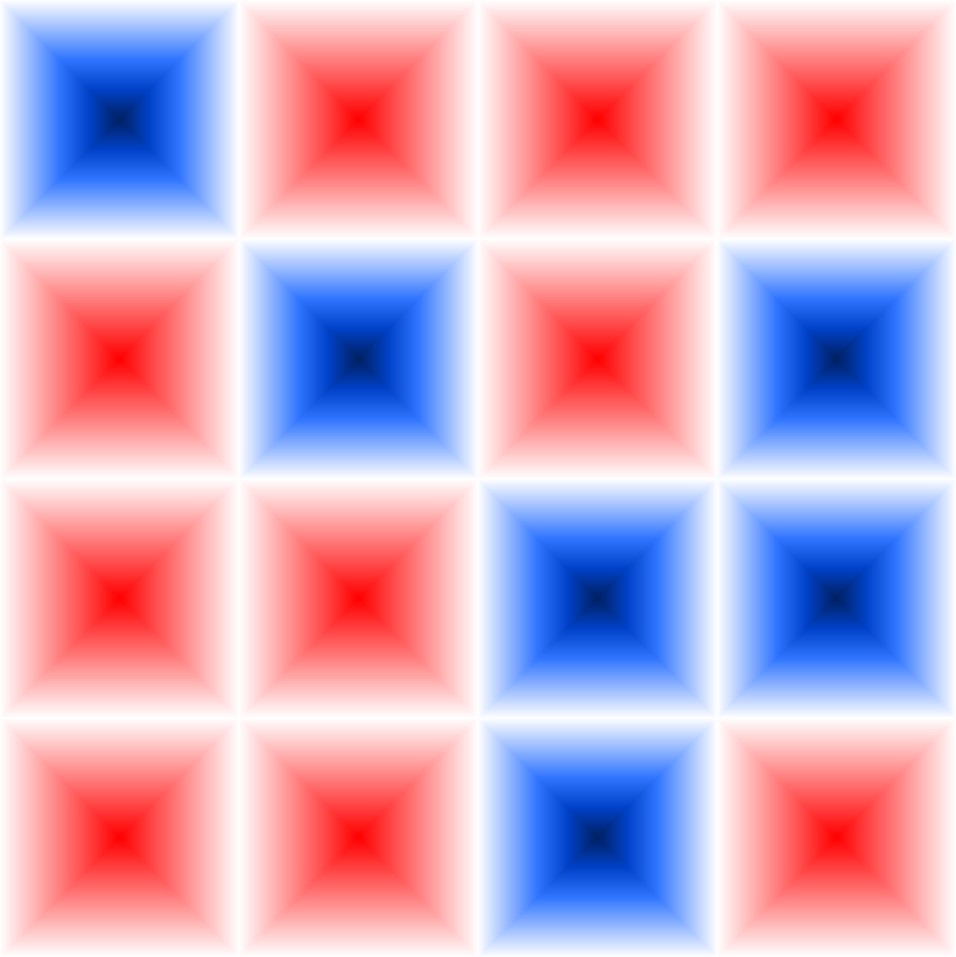}
			\caption{Illustration of $\cX$ and an example utility function $u$ for $d=2$. Red shows preference for $1$, blue shows preference for $0$, and darker shades correspond to more intense preference. (The gradients are rectangular to match the $L_\infty$ norm, so, strangely enough, the misleading X pattern is an optical illusion.)}
			\label{fig:grid}
		\end{figure}

		We claim that the utility function of Equation~\eqref{eq:util} is indeed $L$-Lipschitz with respect to any $L_p$ norm. This is because for any cube $c_i$, and for any $x, x' \in c_i$, we have
		\begin{align*}
		\left|u(x,y_i) - u(x',y_i)\right| &= L\left|\|x - \mu_i\|_{\infty} - \|x' - \mu_i\|_{\infty}\right|\\
		&\leq L\|x - x'\|_{\infty}
		\leq L\|x - x'\|_{p}.
		\end{align*}
		Moreover, for the other outcome, we have $u(x, \neg y_i) = u(x', \neg y_i) = 0$. It follows that $u$ is $L$-Lipschitz within every cube. At the boundary of the cubes, the utility for any outcome is $0$, and hence $u$ is also continuous throughout $\cX$. Because it is piecewise Lipschitz and continuous, $u$ must be $L$-Lipschitz throughout $\cX$, with respect to any $L_p$ norm.

		Next, let $\mathcal{D}$ be an arbitrary deterministic algorithm that extends classifiers on a sample to $\cX$. We draw the sample $S$ of size $m/2$ from $P$ without replacement. Consider the distribution over favorites of individuals in $S$. Each individual in $S$ has a favorite that is sampled independently from Bernoulli$(1/2)$. Hence, by Hoeffding's inequality, the fraction of individuals in $S$ with a favorite of $0$ is between $\frac{1}{2} - \epsilon$ and $\frac{1}{2} + \epsilon$ with probability at least $1 - 2\exp(-m \epsilon^2)$. The same holds simultaneously for the fraction of individuals with favorite $1$.

		Given the sample $S$ and the utility function $u$ on the sample (defined by the instantiation of their favorites), consider the classifier $h_{u,S}$, which maps each individual $\mu_i$ in the sample $S$ to his favorite $y_i$. This classifier is clearly EF on the sample. Consider the extension $h_{u,S}^\mathcal{D}$ of $h_{u,S}$ to the whole of $\cX$ as defined by algorithm $\mathcal{D}$. Define two sets $Z_0$ and $Z_1$ by letting $Z_y = \{\mu_j \notin S \ | \ h_{u,S}^\mathcal{D}(\mu_j) = y\}$, and let $y_*$ denote an outcome that is assigned to at least half of the out-of-sample centers, i.e., an outcome for which $|Z_{y_*}| \geq |Z_{\neg y_*}|$. Furthermore, let $\theta$ denote the fraction of out-of-sample centers assigned to $y_*$. Note that, since $|S| = m/2$, the number of out-of-sample centers is also exactly $m/2$. This gives us $|Z_{y_*}| = \theta \frac{m}{2}$, where $\theta \geq \frac{1}{2}$.

		Consider the distribution of favorites in $Z_{y_*}$ (these are independent from the ones in the sample since $Z_{y_*}$ is disjoint from $S$). Each individual in this set has a favorite sampled independently from Bernoulli$(1/2)$. Hence, by Hoeffding's inequality, the fraction of individuals in $Z_{y_*}$ whose favorite is $\neg y_*$ is at least $\frac{1}{2} - \epsilon$ with probability at least $1 - \exp(-\frac{m}{2} \epsilon^2)$. We conclude that with a probability at least
		$1 - 2\exp(-m \epsilon^2) - \exp(-\frac{m}{2} \epsilon^2)$,
		the sample $S$ and favorites (which define the utility function $u$) are such that: (i)~the fraction of individuals in $S$ whose favorite is $y \in \{0,1\}$ is between $\frac{1}{2} - \epsilon$ and $\frac{1}{2} + \epsilon$, and (ii)~the fraction of individuals in $Z_{y_*}$ whose favorite is $\neg y_*$ is at least $\frac{1}{2} - \epsilon$.

		We now show that for such a sample $S$ and utility function~$u$, $h_{u,S}^\mathcal{D}$ cannot be $(\alpha, \beta)$-EF with respect to $P$ for any $\alpha < 1/25$ and $\beta < L/8$. To this end, sample $x$ and $x'$ from $P$. One scenario where $x$ envies $x'$ occurs when (i)~the favorite of $x$ is $\neg y_*$, (ii)~$x$ is assigned to $y_*$, and (iii)~$x'$ is assigned to $\neg y_*$. Conditions (i) and (ii) are satisfied when $x$ is in $Z_{y_*}$ and his favorite is $\neg y_*$. We know that at least a $\frac{1}{2} - \epsilon$ fraction of the individuals in $Z_{y_*}$ have the favorite $\neg y_*$. Hence, the probability that conditions (i) and (ii) are satisfied by $x$ is at least $(\frac{1}{2} - \epsilon)|Z_{y_*}|\frac{1}{m} = (\frac{1}{2} - \epsilon)\frac{\theta}{2}$.
		Condition (iii) is satisfied when $x'$ is in $S$ and has favorite $\neg y_*$ (and hence assigned $\neg y_*$), or, if $x'$ is in $Z_{\neg y_*}$. We know that at least a $\left(\frac{1}{2} - \epsilon\right)$ fraction of the individuals in $S$ have the favorite $\neg y_*$. Moreover, the size of $Z_{\neg y_*}$ is $(1-\theta)\frac{m}{2}$. So, the probability that condition (iii) is satisfied by $x'$ is at least
		\[\frac{\left(\frac{1}{2} - \epsilon\right)|S| + |Z_{\neg y_*}|}{m} = \frac{1}{2}\left(\frac{1}{2} - \epsilon\right) + \frac{1}{2}(1-\theta).\]

		Since $x$ and $x'$ are sampled independently, the probability that all three conditions are satisfied is at least \[\left(\frac{1}{2} - \epsilon\right)\frac{\theta}{2} \cdot  \left[\frac{1}{2}\left(\frac{1}{2} - \epsilon\right) + \frac{1}{2}(1-\theta)\right].\] This expression is a quadratic function in $\theta$, that attains its minimum at $\theta = 1$ irrespective of the value of $\epsilon$. Hence, irrespective of $\mathcal{D}$, this probability is at least $\left[\frac{1}{2}\left(\frac{1}{2} - \epsilon\right)\right]^2$. For concreteness, let us choose $\epsilon$ to be $1/10$ (although it can be set to be much smaller). On doing so, we have that the three conditions are satisfied with probability at least $1/25$. And when these conditions are satisfied, we have $u(x, h_{u,S}^\mathcal{D}(x)) = 0$ and $u(x, h_{u,S}^\mathcal{D}(x')) = Ls/2$, i.e., $x$ envies $x'$ by $Ls/2 = L/8$. This shows that, when $x$ and $x'$ are sampled from $P$, with probability at least $1/25$, $x$ envies $x'$ by $L/8$. We conclude that with probability at least 	$1 - 2\exp(-m/100) - \exp(-m/200)$ jointly over the selection of the utility function $u$ and the sample $S$, the extension of $h_{u,S}$ by $\mathcal{D}$ is not $(\alpha,\beta)$-EF with respect to $P$ for any $\alpha < 1/25$ and $\beta < L/8$.

		To convert the joint probability into expected cost in the game, note that for two discrete, independent random variables $X$ and $Y$, and for a Boolean function $\mathcal{E}(X,Y)$, it holds that
		\begin{equation}
		\label{eq:exp}
		\operatorname{Pr}_{X,Y}(\mathcal{E}(X,Y) = 1) = \mathbb{E}_X\left[\operatorname{Pr}_Y(\mathcal{E}(X,Y) = 1)\right].
		\end{equation}
		Given sample $S$ and utility function $u$, let $\mathcal{E}(u, S)$ be the Boolean function that equals $1$ if and only if the extension of $h_{u,S}$ by $\mathcal{D}$ is not $(\alpha,\beta)$-EF with respect to $P$ for any $\alpha < 1/25$ and $\beta < L/8$. From Equation~\eqref{eq:exp},
		$\operatorname{Pr}_{u,S}(\mathcal{E}(u,S) = 1)$ is equal to $\mathbb{E}_u\left[\operatorname{Pr}_S(\mathcal{E}(u,S) = 1)\right]$. The latter term is exactly the expected value of the cost, where the expectation is taken over the randomness of $u$. It follows that the expected cost of (any) $\mathcal{D}$ with respect to the chosen distribution over utilities is at least $1 - 2\exp(-m/100) - \exp(-m/200)$.
	%
	\end{proof}

\section{Appendix for Section~\ref{sec:mixtures}} \label{app:mixtures}

This section is devoted to proving our main result:

\noindent\textbf{Theorem~\ref{thm:mixtureGeneralize}.}
	\emph{Suppose $\cG$ is a family of deterministic classifiers of Natarajan
	dimension $d$, and let $\cH = \cH(\cG,m)$ for $m\in \mathbb{N}$. For any distribution
	$P$ over $\cX$, $\gamma>0$, and $\delta > 0$, if $S = \{(x_i, x'_i)\}_{i=1}^n$
	is an i.i.d.~sample of pairs drawn from $P$ of size \[n \geq
	O\left(\frac{1}{\gamma^2}\left(dm^2 \log \frac{dm|\cY|\log(m|\cY|/\gamma)}{\gamma} + \log
	\frac{1}{\gamma}\right)\right),\] then with probability at least $1-\delta$, every classifier $h
	\in \cH$ that is $(\alpha,\beta)$-pairwise-EF on $S$ is also
	$(\alpha+7\gamma, \beta+4\gamma)$-EF on $P$.}

We start with an observation that will be required later.

\begin{lem} \label{lem:G2}
	Let $\cG = \{g : \cX \to \cY\}$ have Natarajan dimension $d$. For $g_1,g_2 \in
	\cG$, let $(g_1,g_2) : \cX \to \cY^2$ denote the function given by
	$(g_1,g_2)(x) = (g_1(x), g_2(x))$ and let $\cG^2 = \{ (g_1,g_2) \,:\, g_1,g_2 \in
	\cG\}$. Then the Natarajan dimension of $\cG^2$ is at most $2d$.
\end{lem}

\begin{proof}
	Let $D$ be the Natarajan dimension of $\cG^2$. Then we know that there exists
	a collection of points $x_1, \dots, x_D \in \cX$ that is shattered by $\cG^2$,
	which means there are two sequences $q_1, \dots, q_n \in \cY^2$ and $q'_1,
	\dots, q'_n \in \cY^2$ such that for all $i$ we have $q_i \neq q'_i$ and for
	any subset $C \subset [D]$ of indices, there exists $(g_1,g_2) \in \cG^2$ such
	that $(g_1,g_2)(x_i) = q_i$ if $i \in C$ and $(g_1,g_2)(x_i) = q'_i$ otherwise.

	Let $n_1 = \sum_{i=1}^D \ind{q_{i1} \neq q'_{i1}}$ and $n_2 = \sum_{i=1}^D
	\ind{q_{i2} \neq q'_{i2}}$ be the number of pairs on which the first and
	second labels of $q_i$ and $q'_i$ disagree, respectively. Since none of the
	$n$ pairs are equal, we know that $n_1 + n_2 \geq D$, which implies that at at
	least one of $n_1$ or $n_2$ must be $\geq D/2$. Assume without loss of generality that $n_1 \geq
	D/2$ and that $q_{i1} \neq q'_{i1}$ for $i = 1, \dots, n_1$. Now consider any
	subset of indices $C \subset [n_1]$. We know there exists a pair of functions
	$(g_1,g_2) \in \cG^2$ with $(g_1,g_2)(x_i)$ evaluating to $q_i$ if $i \in C$
	and $q'_i$ if $i \not \in C$. But then we have $g_1(x_i) = q_{i1}$ if $i \in
	C$ and $g_1(x_i) = q'_{i1}$ if $i \not \in C$, and $q_{i1} \neq q'_{i1}$ for all
	$i \in [n_1]$. It follows that $\cG$ shatters $x_1, \dots, x_{n_1}$, which
	consists of at least $D/2$ points. Therefore, the Natarajan dimension of
	$\cG^2$ is at most $2d$, as required.
\end{proof}

We now turn two the theorem's two main steps, presented in the following two lemmas.

\begin{lem} \label{lem:finiteGeneralize}
	Let $\cH \subset \{h : \cX \to \Delta(\cY)\}$ be a finite family of
	classifiers. For any $\gamma > 0$, $\delta > 0$, and $\beta \geq 0$ if $S =
	\{(x_i, x'_i)\}_{i=1}^n$ is an i.i.d.~sample of pairs from $P$ of size $n \geq
	\frac{1}{2\gamma^2}\ln \frac{|\cH|}{\delta}$, then with probability at least
	$1-\delta$, every $h \in \cH$ that is $(\alpha,\beta)$-pairwise-EF on $S$
	(for any $\alpha)$ is also $(\alpha + \gamma, \beta)$-EF on $P$.
\end{lem}

\begin{proof}
	Let $f(x,x',h) = \ind{u(x,h(x)) < u(x,h(x')) - \beta}$ be the indicator that
	$x$ is envious of $x'$ by at least $\beta$ under classifier $h$. Then $f(x_i,
	x'_i, h)$ is a Bernoulli random variable with success probability
	$\expect_{x,x'\sim P}[f(x,x',h)]$. Applying Hoeffding's inequality to any
	fixed hypothesis $h \in \cH$ guarantees that $\prob_S(\expect_{x,x'\sim
		P}[f(x,x',h)] \geq \frac{1}{n}\sum_{i=1}^n f(x_i, x'_i, h) + \gamma) \leq
	\exp(-2n\gamma^2)$. Therefore, if $h$ is $(\alpha,\beta)$-EF on $S$,
	then it is also $(\alpha+\gamma, \beta)$-EF on $P$ with probability at least $1 - \exp(-2n\gamma^2)$. Applying the union
	bound over all $h \in \cH$ and using the lower bound on $n$ completes the
	proof.
\end{proof}

Next, we show that $\cH(\cG, m)$ can be covered by a finite subset. Since each
classifier in $\cH$ is determined by the choice of $m$ functions from $\cG$ and
mixing weights $\weights \in \Delta_m$, we will construct finite covers of $\cG$
and $\Delta_m$. Our covers $\hat \cG$ and $\hat \Delta_m$ will guarantee that
for every $g \in \cG$, there exists $\hat g \in \hat \cG$ such that $\prob_{x
\sim P}(g(x) \neq \hat g(x)) \leq \gamma/m$. Similarly, for any mixing weights
$\weights \in \Delta_m$, there exists $\hat \weights \in \Delta_m$ such that
$\norm{\weights - \hat \weights}_1 \leq \gamma$. If $h \in \cH(\cG,m)$ is the
mixture of $g_1, \dots, g_m$ with weights $\weights$, we let $\hat h$ be the
mixture of $\hat g_1, \dots, \hat g_m$ with weights $\hat \weights$. This
approximation has two sources of error: first, for a random individual $x \sim
P$, there is probability up to $\gamma$ that at least one $g_i(x)$ will disagree
with $\hat g_i(x)$, in which case $h$ and $\hat h$ may assign completely
different outcome distributions. Second, even in the high-probability event that
$g_i(x) = \hat g_i(x)$ for all $i \in [m]$, the mixing weights are not
identical, resulting in a small perturbation of the outcome distribution
assigned to $x$.

\begin{lem} \label{lem:finiteCover}
	Let $\cG$ be a family of deterministic classifiers with Natarajan dimension $d$,
	and let $\cH = \cH(\cG, m)$ for some $m\in\mathbb{N}$. For any $\gamma > 0$, there
	exists a subset $\hat \cH \subset \cH$ of size $O\bigl(\frac{(dm|\cY|^2 \log(m
		|\cY| / \gamma))^{dm}}{\gamma^{(d+1)m}}\bigr)$ such that for every $h \in \cH$
	there exists $\hat h \in \cH$ satisfying:
	\begin{enumerate}
		\item $\prob_{x \sim P}( \norm{h(x) - \hat h(x)}_1 > \gamma) \leq \gamma$.
		\item If $S$ is an i.i.d.~sample of individuals of size
		$O(\frac{m^2}{\gamma^2}(d \log |\cY| + \log \frac{1}{\delta}))$ then w.p.
		$\geq 1-\delta$, we have $\norm{h(x) - \hat h(x)}_1 \leq \gamma$ for all but
		a $2\gamma$-fraction of $x \in S$.
	\end{enumerate}
\end{lem}

\begin{proof}
	As described above, we begin by constructing finite covers of $\Delta_m$ and
	$\cG$. First, let $\hat \Delta_m \subset \Delta_m$ be the set of distributions
	over $[m]$ where each coordinate is a multiple of $\gamma/m$. Then we have
	$|\hat \Delta_m| = O( (\frac{m}{\gamma})^m)$ and for every $p \in \Delta_m$,
	there exists $q \in \hat \Delta_m$ such that $\norm{p - q}_1 \leq \gamma$.

	In order to find a small cover of $\cG$, we use the fact that it has low
	Natarajan dimension. This implies that the number of effective functions in
	$\cG$ when restricted to a sample $S'$ grows only polynomially in the size of
	$S'$. At the same time, if two functions in $\cG$ agree on a large sample,
	they will also agree with high probability on the distribution.

	Formally, let $S'$ be an i.i.d.~sample drawn from $P$ of size
	$O(\frac{m^2}{\gamma^2}d\log|\cY|)$, and let $\hat \cG = \cG\resto_{S'}$ be any
	minimal subset of $\cG$ that realizes all possible labelings of $S'$ by
	functions in $\cG$. We now argue that with probability 0.99, for every $g \in
	\cG$ there exists $\hat g \in \hat \cG$ such that $\prob_{x \sim P}(g(x) \neq
	\hat g(x)) \leq \gamma/m$. For any pair of functions $g,g' \in \cG$, let $(g,g') :
	\cX \to \cY^2$ be the function given by $(g,g')(x) = (g(x), g'(x))$, and let
	$\cG^2 = \{(g,g') \,:\, g,g' \in \cG\}$. The Natarajan dimension of $\cG^2$ is
	at most $2d$ by Lemma~\ref{lem:G2}. Moreover, consider the loss $c : \cG^2 \times \cX \to \{0,1\}$ given by $c(g,g',x) =
	\ind{g(x) \neq g'(x)}$. Applying Lemma~\ref{lem:natarajanAgnostic} with the
	chosen size of $|S'|$ ensures that with probability at least $0.99$ every pair
	$(g,g') \in \cG^2$ satisfies \[\left|\expect_{x \sim P}[c(g,g',x)] -
	\frac{1}{|S'|} \sum_{x \in S'} c(g,g',x)\right| \leq \frac{\gamma}{m}.\] By the definition
	of $\hat \cG$, for every $g \in \cG$, there exists $\hat g \in \hat \cG$ for
	which $c(g,\hat g,x) = 0$ for all $x \in S'$, which implies that $\prob_{x \sim
		P}(g(x) \neq \hat g(x)) \leq \gamma/m$.

	Using Lemma~\ref{lem:natarajan} to
	bound the size of $\hat \cG$, we have that \[|\hat \cG| \leq |S'|^d |\cY|^{2d} =
	O\left(\left(\frac{m^2}{\gamma^2}d|\cY|^2\log|\cY|\right)^d\right).\] Since this
	construction succeeds with non-zero probability, we are guaranteed that such a
	set $\hat \cG$ exists. Finally, by an identical uniform convergence argument,
	it follows that if $S$ is a fresh i.i.d.~sample of the size given in Item 2 of the lemma's
	statement, then, with probability at least $1-\delta$, every $g$ and $\hat g$ will disagree on at most a
	$2\gamma/m$-fraction of $S$, since they disagree with probability at most
	$\gamma/m$ on $P$.

	Next, let $\hat \cH = \{ h_{\vec g, \weights} \, : \, \vec g \in \hat G^m, \weights
	\in \hat \Delta_m\}$ be the same family as $\cH$, except restricted to choosing
	functions from $\hat \cG$ and mixing weights from $\hat \Delta_m$. Using the
	size bounds above and the fact that ${N \choose m} = O((\frac{N}{m})^m)$, we
	have that \[|\hat \cH| = {|\hat \cG| \choose m} \cdot |\hat \Delta_m| =
	O\left(\frac{(dm^2|\cY|^2 \log(m |\cY| /
		\gamma))^{dm}}{\gamma^{(2d+1)m}}\right).\]

	Suppose that $h$ is the mixture of $g_1, \dots, g_m \in \cG$ with weights
	$\weights \in \Delta_m$. Let $\hat g_i$ be the approximation to $g_i$ for each
	$i$, let $\hat \weights \in \hat \Delta_m$ be such that $\norm{\weights - \hat
		\weights}_1 \leq \gamma$, and let $\hat h$ be the random mixture of $\hat g_1,
	\dots, \hat g_m$ with weights $\hat \weights$. For an individual $x$ drawn from
	$P$, we have $g_i(x) \neq \hat g_i(x)$ with probability at most $\gamma/m$,
	and therefore they all agree with probability at least $1-\gamma$. When this
	event occurs, we have $\norm{h(x) - \hat h(x)}_1 \leq \norm{\weights - \hat
		\weights}_1 \leq \gamma$.

	The second part of the claim follows by similar reasoning, using the fact that
	for the given sample size $|S|$, with probability at least $1-\delta$, every
	$g \in \cG$ disagrees with its approximation $\hat g \in \hat \cG$ on at most
	a $2\gamma/m$-fraction of $S$. This means that $\hat g_i(x) = g_i(x)$ for all $i \in
	[m]$ on at least a $(1-2\gamma)$-fraction of the individuals $x$ in $S$. For these
	individuals, $\norm{h(x) - \hat h(x)}_1 \leq \norm{\weights - \hat \weights}_1 \leq \gamma$.
\end{proof}

Combining the generalization guarantee for finite families given in
Lemma~\ref{lem:finiteGeneralize} with the finite approximation given in
Lemma~\ref{lem:finiteCover}, we are able to show that envy-freeness also
generalizes for $\cH(\cG,m)$.

\begin{proof}[Proof of Theorem~\ref{thm:mixtureGeneralize}]
	Let $\hat \cH$ be the finite approximation to $\cH$ constructed in
	Lemma~\ref{lem:finiteCover}. If the sample is of size $|S| =
	O(\frac{1}{\gamma^2}(dm \log(dm|\cY|\log|\cY|/\gamma) + \log
	\frac{1}{\delta}))$, we can apply Lemma~\ref{lem:finiteGeneralize} to this
	finite family, which implies that for any $\beta' \geq 0$, with probability at
	least $1-\delta/2$ every $\hat h \in \hat \cH$ that is
	$(\alpha',\beta')$-pairwise-EF on $S$ (for any $\alpha'$) is also
	$(\alpha'+\gamma, \beta')$-EF on $P$. We apply this lemma with $\beta' = \beta +
	2\gamma$. Moreover, from Lemma~\ref{lem:finiteCover}, we know that if $|S| =
	O(\frac{m^2}{\gamma^2}(d \log|\cY| + \log \frac{1}{\delta}))$, then with
	probability at least $1-\delta/2$, for every $h \in \cH$, there exists  $\hat
	h \in \hat \cH$ satisfying $\norm{h(x) - \hat h(x)}_1 \leq \gamma$ for all but a
	$2\gamma$-fraction of the individuals in $S$. This implies that on all but at
	most a $4\gamma$-fraction of the pairs in $S$, $h$ and $\hat h$ satisfy this
	inequality for both individuals in the pair. Assume these high probability
	events occur. Finally, from Item 1 of the lemma we have that $\prob_{x_1,x_2
	\sim P}(\max_{i=1,2}\norm{h(x_i) - \hat h(x_i)}_1 > \gamma) \leq 2\gamma$.

	Now let $h \in \cH$ be any classifier that is
	$(\alpha,\beta)$-pairwise-EF on $S$. Since the utilities are in $[0,1]$ and
	$\max_{x=x_i,x_i'} \norm{h(x) - \hat h(x)}_1 \leq \gamma$ for all but a
	$4\gamma$-fraction of the pairs in $S$, we know that $\hat h$ is $(\alpha +
	4\gamma, \beta + 2\gamma)$-pairwise-EF  on $S$. Applying the envy-freeness
	generalization guarantee (Lemma~\ref{lem:finiteGeneralize}) for $\hat \cH$, it follows that $\hat h$ is also
	$(\alpha + 5\gamma, \beta + 2\gamma)$-EF on $P$. Finally, using the
	fact that \[\prob_{x_1,x_2 \sim P}\left(\max_{i=1,2} \norm{h(x_i) - \hat h(x_i)}_1 >
	\gamma\right) \leq 2\gamma,\] it follows that $h$ is $(\alpha + 7\gamma, \beta +
	4\gamma)$-EF on $P$.
\end{proof}

It is worth noting that the (exponentially large) approximation $\hat \cH$ is
only used in the generalization analysis; importantly, an ERM algorithm need not construct
it.

\section{Appendix for Section~\ref{sec:expts}}	\label{app:expts}

Here we describe details of the transformation of the optimization problem from~\eqref{eqn:greedy-opt} to~\eqref{eqn:relaxed-opt}. Firstly, softening constraints of~\eqref{eqn:greedy-opt} with slack variables, we obtain
\begin{align*}
&\min_{g_k \in \cG, \xi \in \mathbb{R}^{n \times n}_{\geq 0}} \quad \sum_{i=1}^n L(x_i, g_k(x_i)) + \lambda \sum_{i \neq j} \xi_{ij}\notag\\
&\qquad\text{s.t.} \quad USF_{ii}^{(k-1)} + \tilde{\weights}_k u(x_i, g_k(x_i)) \geq USF_{ij}^{(k-1)} + \tilde{\weights}_k u(x_i, g_k(x_j)) - \xi_{ij} \quad \forall (i,j).
\end{align*}
Here, $\xi_{ij}$ basically captures how much $i$ envies $j$ under the selected assignments (note that, $\xi_{ij}$ is $0$ if the pair is non-envious, so that the algorithm does not go increasing negative envy at the cost of positive envy for someone else).
Plugging in optimal values of the slack variables, we obtain
\begin{align}	\label{eqn:softened-opt}
&\min_{g_k \in \cG} \quad \sum_{i=1}^n L(x_i, g_k(x_i)) \notag\\
&\qquad\quad+ \lambda \sum_{i \neq j} \max\left(USF_{ij}^{(k-1)} + \tilde{\weights}_k u(x_i, g_k(x_j)) - USF_{ii}^{(k-1)} - \tilde{\weights}_k u(x_i, g_k(x_i)), 0\right).
\end{align}
Next, we perform convex relaxation of different components of this objective function. For this, let's observe the term $L(x_i, g_k(x_i))$. And, let $\vec{w}$ denote the parameters of $g_k$. By definition, we have
$$w_{g_k(x_i)}^\tp x_i \geq w_{y'}^\tp x_i$$
for any $y' \in \cY$. This implies that
\begin{align*}
L(x_i, g_k(x_i)) &\leq L(x_i, g_k(x_i)) + w_{g_k(x_i)}^\tp x_i - w_{y'}^\tp x_i\\
&\leq \max_{y \in \cY}\left\{L(x_i, y) + w_{y}^\tp x_i - w_{y'}^\tp x_i\right\},
\end{align*}
giving us a convex upper bound on the loss $L(x_i, g_k(x_i))$. As this holds for any $y' \in \cY$, we choose $y' = y_i$ as defined in the main body, since it leads to the lowest achievable loss value. Therefore, we have
$$L(x_i, g_k(x_i)) \leq \max_{y \in \cY}\left\{L(x_i, y) + w_{y}^\tp x_i - w_{y_i}^\tp x_i\right\}.$$
This right hand side is basically an upper bound which apart from encouraging $\vec{w}$ to have the highest dot product with $x_i$ at $y_i$, also penalizes if the margin by which this is higher is not enough (where the margin depends on other losses $L(x_i, y)$). This surrogate loss is very similar to multi-class support vector machines. We perform similar relaxations for the other two components of the objective function. In particular, for the $u(x_i, g_k(x_i))$ term, we have
$$-u(x_i, g_k(x_i)) \leq \max_{y \in \cY}\left\{-u(x_i, y) + w_y^\tp x_i - w_{b_i}^\tp x_i\right\},$$
where $b_i$ is as defined in the main body. Finally, for the remaining term, we have
$$u(x_i, g_k(x_j)) \leq \max_{y \in \cY}\left\{u(x_i, y) + w_y^\tp x_j - w_{s_i}^\tp x_j\right\},$$
where $s_i$ is as defined in the main body\footnote{Note that, instead of using $s_i$, an alternative to use in this equation is $b_j$. In particular, for a pair $(i,j)$, using $s_i$ encourages the assignment to give $i$ their favorite outcome while $j$ the outcome that $i$ likes the least (and hence causing $i$ to envy $j$ as less as possible), while using $b_j$ encourages the assignment to give both $i$ and $j$ their favorite outcomes (pushing the assignment to just give everyone their favorite outcomes).}. On plugging in the convex surrogates of all three terms in Equation~\eqref{eqn:softened-opt}, we obtain the optimization problem~\eqref{eqn:relaxed-opt}.

\end{document}